\documentclass[conference]{IEEEtran}
\usepackage{times}

\usepackage[markup=underlined, final]{changes}
\usepackage[numbers]{natbib}
\usepackage{multicol}
\usepackage[bookmarks=true]{hyperref}
\usepackage{lipsum}
\usepackage{graphicx}
\usepackage{gensymb}
\usepackage{amsthm}
\usepackage{amsmath}
\usepackage{amssymb}
\usepackage{mathtools}
\usepackage{diagbox}
\usepackage[left=0.625in,right=0.625in,top=0.75in,bottom=1in]{geometry}
\DeclarePairedDelimiter{\ceil}{\lceil}{\rceil}
\newcommand{\RN}[1]{%
  \textup{\uppercase\expandafter{\romannumeral#1}}%
}
\usepackage{bm}
\usepackage{booktabs}
\usepackage[linesnumbered,ruled,vlined]{algorithm2e}
\usepackage{xcolor}
\DeclareMathOperator*{\argmax}{arg\,max}

\ifCLASSOPTIONcompsoc
    \usepackage[caption=false, font=normalsize, labelfont=sf, textfont=sf]{subfig}
\else
    \usepackage[caption=false, font=footnotesize]{subfig}
\fi

\graphicspath{{./figs/}}

\newtheorem{definition}[]{Definition}
\newtheorem{assumption}{Assumption}
\newtheorem{lemma}{Lemma}
\newtheorem{theorem}{Theorem}
\newtheorem{problem}{Problem}
\newtheorem{policy}{Policy}

\begin{document}

\title{Pareto Monte Carlo Tree Search \\for Multi-Objective Informative Planning}



\IEEEoverridecommandlockouts
\author{Weizhe Chen, Lantao Liu
\thanks{W. Chen and L. Liu are with 
the School of Informatics, Computing, and Engineering  at Indiana University, Bloomington, IN 47408, USA. E-mail:
{\tt\small \{chenweiz, lantao\}@iu.edu}.
}
}


%

\maketitle

\begin{abstract}
In many environmental monitoring scenarios,  the sampling robot needs to simultaneously explore the environment and exploit features of interest with limited time.
We present an anytime multi-objective informative planning method called Pareto Monte Carlo tree search which allows the robot to handle potentially competing objectives such as exploration versus exploitation.
The method produces optimized decision solutions for the robot based on its knowledge (estimation) of the environment state, leading to better adaptation to environmental dynamics.
We provide algorithmic analysis on the critical tree node selection step and show that the number of times choosing sub-optimal nodes is logarithmically bounded and the search result converges to the optimal choices at a polynomial rate.
\end{abstract}

\IEEEpeerreviewmaketitle

\section{INTRODUCTION}
There is an increasing need that mobile robots are tasked to gather information from our environment.
Navigating robots to collect samples with the largest amount of information is called informative planning~\cite{binney13, Meliou07, Singh2007}. 
The basic idea is to maximize information gain using information-theoretic methods, where the information gain (or {\em informativeness}) is calculated from the estimation uncertainty based on some prediction models such as the set of Gaussian systems. 
Informative planning is challenging due to the large and complex searching space.
Such problems have been shown to be NP-hard~\cite{singh2009efficient} or PSPACE-hard~\cite{reif1979complexity} depending on the form of objective functions and the corresponding searching space.


Existing informative planning work is heavily built upon information-theoretic framework.
The produced solutions tend to guide the robot to explore the uncertain or unknown areas and reduce the estimation uncertainty (e.g., the entropy or mutual information based efforts). 
Aside from the estimation uncertainty, there are many other estimation properties that we are interested in.
For example, when monitoring some biological activity in a lake, scientists are interested in areas with high concentration -- known as \textit{hotspots} -- of biological activity~\cite{mccammon2018}.
Similarly,  plume tracking is useful for detection and description of oil spills, thermal vents, and harmful algal blooms~\cite{susca2008}.
In these tasks, the robot needs to explore the environment to discover hotspots first, and then visit different areas with differing effort (frequency) to obtain the most valuable samples and catch up with the (possibly) spatiotemporal environmental dynamics.

\begin{figure}[htbp] \vspace{0pt}
    \centering
    \subfloat[Bathymetry\label{fig:newzealand}]{\includegraphics[width=0.48\linewidth]{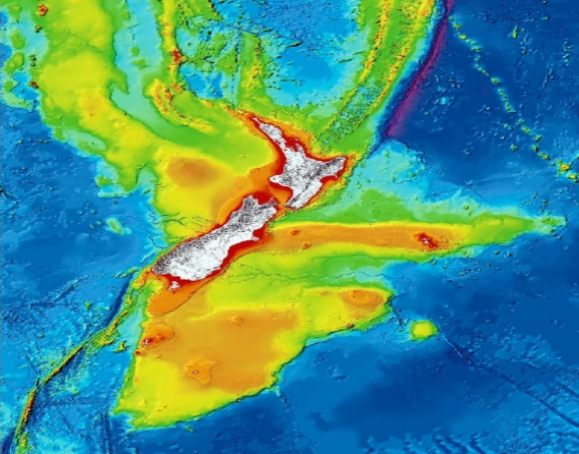}}
    \ \ 
    \subfloat[Example trajectory\label{fig:pull_figure}]{\includegraphics[width=0.48\linewidth]{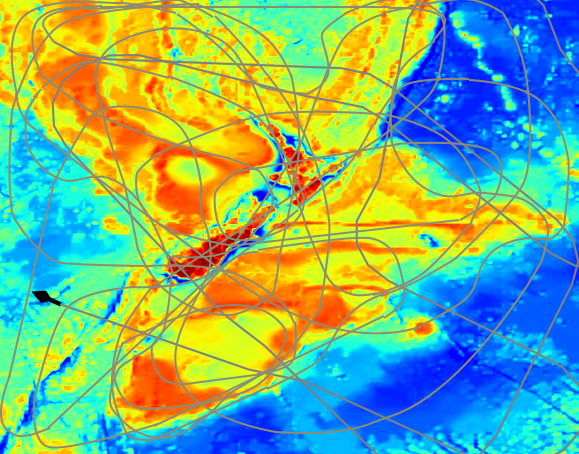}}
    \caption{Illustrative multi-objective informative planning trajectory for modeling the Bathymetry of New Zealand's Exclusive Economic Zone. The goal is to explore the environment while visiting areas with high value and variability more frequently. (Source: NIWA) \vspace{-10pt}
    }  \vspace{-12pt}
\end{figure}

Since exploration and exploitation cannot be achieved simultaneously most of the time,  different planning objectives may lead to distinct motion trajectories.
The problem falls within the spectrum of multi-objective optimization in which the optima are not unique.
In fact, continua of optimal solutions are possible.
The goal is to find the Pareto optimal set, or so called Pareto front.
The solutions in the Pareto optimal set cannot be improved for any objective without hurting other objectives. 
A straightforward solution to the multi-objective problem is to convert it into a single-objective problem by linear scalarization (weighted sum of objectives is one of the linear scalarization method).
However, linear scalarization might be impossible, infeasible, or undesirable for the following reasons.
First of all, objectives might be conflicting.
Taking algal bloom monitoring as an example, high-concentration areas and the areas with high temporal variability (i.e., spreading dynamics) are both important.
We expect the robot to visit these areas more frequently, but these two types of areas are not necessarily in the same direction.
Furthermore, the weights among different objectives are hard to determine in the linear scalarization approaches.
Last but not least, linear scalarization based approaches fail to discover the solutions in the non-convex regions of the Pareto front.

This paper presents the following contributions:
\begin{itemize}
    \item 
    Different from existing informative planning approaches where information-seeking is the main objective, we propose a new generic planning method that optimizes over multiple (possibly) competing objectives and constraints.
    \item We incorporate the Pareto optimization into the  Monte Carlo tree search process and further design an anytime and non-myopic planner for in-situ decision-making. 
    \item We provide in-depth algorithmic analysis which reveals bounding and converging behaviors of the search process.
    \item  We perform thorough simulation evaluations with both synthetic and real-world data. Our results show that the robot exhibits desired composite behaviors which are optimized from corresponding hybrid objectives. 

\end{itemize}

\section{RELATED WORK}


Informative planning maximizes the collected information (informativeness) by exploring (partially) unknown environment during its sampling process~\cite{binney13, Low2009thesis}. 
Comparing with the lawnmower based sweeping style sampling mechanism which focuses on spatial resolution, the informative planning method tends to achieve the spatial coverage quickly with the least estimation  uncertainty~\cite{Singh2007}. 
Due to these reasons, the information planning has been widely used for the spatiotemporal environmental monitoring. 
To explore and learn the environment model, a commonly-used approach in spatial statistics is the Gaussian Process Regression~\cite{Rasmussen2005}.
Built on the learned environmental model, 
path and motion control can be carried out which is a critical ability for autonomous robots operating in unstructured  environments~\cite{Guestrin2003, leonard2010coordinated}.

Representative informative planning approaches include, e.g., algorithms based on a recursive-greedy style~\cite{Meliou07, Singh2007} where the informativeness is generalized as submodular function and a sequential-allocation mechanism is designed in order to obtain subsequent waypoints. 
This recursive-greedy framework has been extended later by incorporating obstacle avoidance and diminishing returns~\cite{binney13}. 
In addition, a differential entropy based framework~\cite{Low2009thesis} was proposed where a batch of waypoints can be obtained through dynamic programming.
Recent work also reveals that online informative planning is possible~\cite{ma2018data}.
The sampling data is thinned based on their contributions learned by a sparse variant of Gaussian Process.
There are also methods optimizing over complex routing constraints (e.g., see~\cite{SolteroSR12,YuSchRus14ICRA}). 

Pareto optimization has been used in designing motion planners to optimize over the length of a path and the probability of collisions~\cite{POMP}. 
Recently, a sampling based method has also been proposed to generate Pareto-optimal trajectories for multi-objective motion planning~\cite{lee2018sampling}.
In addition, 
multi-robot coordination also benefits from multi-objective optimization. The goals of different robots are simultaneously optimized~\cite{ghrist2004pareto, lavalle1998motion}.
To balance the operation cost and the travel discomfort experienced by users, the multi-objective fleet routing algorithms compute the Pareto-optimal fleet operation plans~\cite{Cp2018MultiObjectiveAO}.

Related work also includes the multi-objective reinforcement learning~\cite{roijers2013survey}.
Particularly, the prior work multi-objective Monte Carlo tree search (MO-MCTS) is closely relevant to our method~\cite{wang2012multi}.
Unfortunately, MO-MCTS is computationally prohibitive and cannot be used for online planning framework.
Vast computational resources are needed in order to maintain a global Pareto optimal set with all the best solutions obtained so far.
In contrast, we develop a framework that maintains a local approximate Pareto optimal set in each node which can be processed in a much faster way. Our approach is also flexible and adaptive with regards to capturing environmental variabilities of different stages. 




\section{PRELIMINARIES}

\begin{figure}[htbp] \vspace{0pt}
    \centering
    \includegraphics[width=1\linewidth]{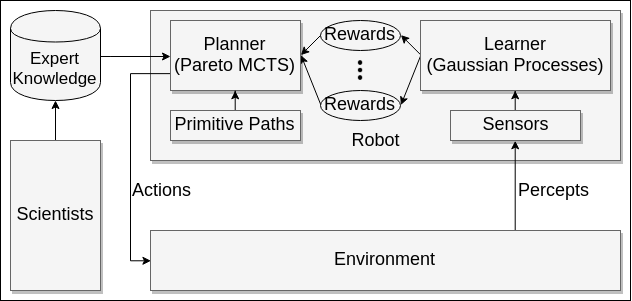}
    \caption{System flow chart of Pareto MCTS for the online planning framework.\vspace{-15pt}}
    \label{fig:system_map} \vspace{-10pt}
\end{figure}


\subsection{Pareto Optimality}
Let $\bm{X}_k$ be a $D$-dimensional reward vector associated with a choice $k$ and $X_{k, d}$ be the $d$-th element.
The $i$-th choice is better than, or \textit{dominating}, another choice $j$, denoted by $i \succ j$ or $j \prec i$, if and only if the following conditions are satisfied:
\begin{enumerate}
    \item Any element of $\bm{X}_i$ is not smaller than the corresponding element in $\bm{X}_j$: \\
    $\forall  d= 1, 2, \dots, D, X_{i,d} \geqslant X_{j,d}$;
    \item At lease one element of $\bm{X}_i$ is larger than the corresponding element in $\bm{X}_j$: \\
    $\exists d \in \{1, 2, \dots, D\} \text{ such that } X_{i,d} > X_{j,d} $.
\end{enumerate}
If only the first condition is satisfied, we say that choice $i$ is \textit{weakly-dominating} choice $j$, denoted by $i \succeq j$ or $j \preceq i$.

In some cases, neither $i \succeq j$ nor $j \succeq i$ hold.
We say that choice $i$ is \textit{incomparable} with choice $j$ ($i || j$) if and only if there exists one dimension $d_1$ such that $X_{i, d_1} > X_{j, d_1}$, and another dimension $d_2$ such that $X_{i, d_2} < X_{j, d_2}$.
Also, we say that choice $i$ is \textit{non-dominated} by choice $j$ ($i \nprec j$ or $j \nsucc i$) if and only if there exists one dimension $d$ such that $X_i^d > X_j^d$.

\subsection{Multi-Objective Informative Planning}
In the general case, \textit{multi-objective informative planning} requires solving the following maximization problem:
\begin{align}
    &\bm{a}^* = \argmax_{\bm{a} \in \mathcal{A}} \left\{ I(\bm{a}), F_1(\bm{a}), \dots, F_{D-1}(\bm{a}) \right\},\\ 
    &\text{ s.t. } C_{\bm{a}} \leq B \notag,
\end{align}
where $\bm{a}$ is a sequence of actions, $\mathcal{A}$ is the space of possible action sequences, $B$ is a budget (e.g. time, energy, memory, or number of iterations), $C_{\bm{a}}$ is the cost of budget, and $I(\bm{a})$ is a function representing the information gathered by executing the action.
$F_d(\bm{a}), d\in \{1, \dots, D-1\}$  are other objective functions defining which types of behaviors are desired.


\section{APPROACH}

\begin{figure*}[htbp] \vspace{-5pt}
    \centering
    \includegraphics[width=0.65\linewidth]{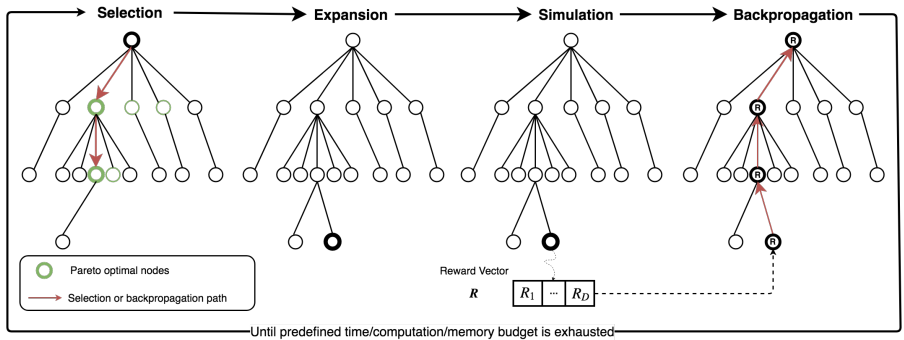}
    \caption{Illustration of main steps of Pareto Monte Carlo tree search (Pareto MCTS).\vspace{-15pt} }
    \label{fig:pareto_mcts} \vspace{-15pt}
\end{figure*}

\subsection{Methodology Overview}

An overview of our method is shown in Fig.~\ref{fig:system_map}.
The planner computes an informative path using the Pareto MCTS algorithm based on  the current estimation of the environment. 
As the robot travels along the path, it continuously receives observations of the target variables \deleted{describing the environmental attributes of interest}.
Then, these newly acquired samples are used to refine its estimation of the environment, which in turn influences the planning in the next round.


The specific form of rewards depend on the applications.
For example, in the informative planning problem, the rewards can be defined as the reduced amount of estimation entropy, accumulated amount of mutual information, etc.
In hotspot exploration and exploitation task, the reward of a given path can be defined as the sum of rewards of all sampling points along the path.
Pareto MCTS searches for the optimal actions for current state until a given time budget is exhausted.
The inputs of Pareto MCTS are a set of available {\em primitive paths} at every possible state and the robot's current knowledge (estimation) of the world. 
It is worth noting that,  the multi-objective framework also allows us to incorporate prior knowledge, as illustrated on the left side of Fig.~\ref{fig:system_map}.


\subsection{Pareto Monte Carlo Tree Search}
The proposed Pareto Monte Carlo tree search (Pareto MCTS) is a planning method for finding Pareto optimal decisions within a given horizon.
The tree is built incrementally in order to expand the most promising subtrees first.
This is done by taking advantage of the information gathered in previous steps.
Each node in the search tree represents a state of the domain (in our context it is a location to be sampled), and the edges represent actions leading to subsequent states.

The framework of Pareto MCTS is outlined in Fig.~\ref{fig:pareto_mcts}. In each iteration, it can be broken down into four main steps:
\begin{enumerate}
    \item \textit{Selection}: starting at the root node, a child node selection policy (described later on) is applied recursively to descend through the tree until an expandable node with unvisited (unexpanded) children is encountered.
    \item \textit{Expansion}: an action is chosen from the aforementioned expandable node.  A child node is constructed according to this action and connected to the expandable node.
    \item \textit{Simulation}: a simulation is run from the new node based on the predefined default policy. In our case, the default policy chooses a random action from all available actions. The reward vector of this simulation is then  returned.
    \item \textit{Backpropagation}: the obtained reward is backed up or backpropagated to each visited node in the selection and expansion steps to update the attributes in those nodes. In our algorithm, each node stores the number of times it has been visited and the cumulative reward obtained by simulations starting from this node.
\end{enumerate}
These steps are repeated until the maximum number of iterations is reached or a given time limit is exceeded.

The most challenging part of designing Pareto MCTS is the selection step.
We want to select the most promising node to expand first in order to improve the searching efficiency.
However, which node is the most promising is unknown so that the algorithm needs to estimate it.
Therefore, in the selection step, one must balance exploitation of the recently discovered most promising child node and exploration of alternatives which may turn out to be a superior choice at later time.

In the case of scalar reward, the most promising node is simply the one with highest expected reward.
However, the reward is herein a vector corresponding to multiple objectives.
In the context of multi-objective optimization, the optima are no longer unique.
Instead, there is a set of optimal choices, each of which is considered equally best.
Below we define the Pareto optimal node set.
The nodes in the Pareto optimal set are incomparable to each other and will not be dominated by any other nodes.

\begin{definition}[Pareto Optimal Node Set]\label{pareto_optimal_set}
Given a set of nodes $\mathcal{V}$, a subset $\mathcal{P}^* \subset \mathcal{V}$ is the Pareto optimal node set, in terms of expected reward, if and only if
\begin{equation*}
    \begin{cases}
        \forall v_i^* \in \mathcal{P}^* \text{ and } \forall v_j \in \mathcal{V}, v_i^* \nprec v_j,\\
        \forall v_i^*, v_j^* \in \mathcal{P}^*, v_i^* || v_j^*.
    \end{cases}
\end{equation*}
\end{definition}

We treat the node selection problem as a multi-objective multi-armed bandit problem.
As shown in Alg.~\ref{alg:pareto_best_child}, the Pareto Upper Confidence Bound (Pareto UCB) for each child node is first computed according to Eq.~\eqref{eq:pareto_ucb}.
Then an approximate Pareto optimal set is built using the resulting Pareto UCB vectors (see the green nodes in Fig.~\ref{fig:pareto_mcts}).
Finally, the best child node is chosen from the Pareto optimal set uniformly at random.
Note that 
the reason for choosing the best child randomly is that the Pareto optimal solutions are considered equally optimal if no preference information is given.
However, if domain knowledge is available or preference is specified, one can choose the most preferable child from the Pareto optimal set.
For example, in the environmental monitoring task, one might expect the robot to explore the environment in the early stage to identify some important areas and spend more effort exploiting these areas in the later stage.
In such case, we can choose the Pareto optimal node with highest information gain in the beginning because the information gain based planning tends to cover the space quickly.
Other types of rewards can be chosen later to concentrate on examining the details locally.

Note that this is different from weighting different objectives and solve a single-objective problem, because choosing a preferred solution from a given set of optimal solutions is often easier than determining the quantitative relationship among different objectives.
For instance, it is difficult to quantify how much the information gain is more important than exploiting high-value areas in the hotspot monitoring task.
However, given several choices, we may pick the most informative choice in the very beginning, 
and another one with highest target value related reward in a later stage.

\begin{algorithm}[htbp]
\label{alg:pareto_best_child}
  \DontPrintSemicolon
    \SetKwProg{Fn}{Function}{}{}
    \Fn{$\text{Search}(s_0)$}{
        create root node $v_0$ with state $s_0$\;
        \While{within computational budget}{
            $v_\text{expandable} \leftarrow \text{ Selection}(v_0)$\;
            $v_\text{new} \leftarrow \text{ Expansion}(v_\text{expandable})$\;
            $\text{RewardVector} \leftarrow \text{ Simulation }(v_\text{new}.s)$\;
            $\text{Backpropagation}(v_\text{new}, \text{RewardVector})$\;
        }
        \Return $\text{MostVisitedChild}(v_0)$\;
    }
    
    \Fn{$\text{Selection}(v)$}{
        \While{$v$ is fully expanded}{
            $v \leftarrow \text{ParetoBestChild}(v)$\;
        }
        \Return $v$\;
    }
    
    \Fn{$\text{ParetoBestChild}(v)$}{
        compute Pareto UCB for each child $k$:
        \begin{equation}\label{eq:pareto_ucb}
        \bm{U}(k) = \frac{v_k.\bm{X}}{v_k.n} + \sqrt{\frac{4\ln{n} + \ln{D}}{2v_k.n}}
        \end{equation}\;
        build approximate Pareto optimal node set $v.\mathcal{P}$ based on $\bm{U}(k)$\;
        choose a child $v_{\text{best}}$ from $v.\mathcal{P}$ uniformly at random\;
        \Return $v_{\text{best}}$\;
    }
 \caption{Pareto MCTS}
Note: We use $v.\text{attribute}$ to represent a node attribute of $v$. In Eq.~\eqref{eq:pareto_ucb}, $v_k.\bm{X}$ is the cumulative reward of the $k$-th child of node $v$, $v_k.n$ is the number of times that the $k$-th child has been visited, and $D$ is the number of objectives. Some standard components, such as Simulation, Expansion, and Backpropagation, are not shown due to the lack of space.
\end{algorithm}

\subsection{Algorithm Analysis}
As mentioned earlier, the node selection is the most critical step and almost determines the performance of the entire framework.
Thus, here we spend effort analyzing this step to better understand its important properties.
Although some (potentially) sub-optimal nodes may be selected inevitably, we show that the number of times choosing a sub-optimal node can be bounded logarithmically in Pareto MCTS.
In addition, we want to know whether this anytime algorithm will converge to the optimal solution if enough time is given and whether a good solution can be returned if it is stopped midway.
To answer this question, we show that the searching result of Pareto MCTS  converges to the Pareto optimal choices at a polynomial rate.

\begin{problem}[Node Selection]\label{node_selection}
Consider a node $v$ with $K$ child nodes in a Pareto Monte Carlo search tree.
At decision step $n$, a D-dimensional random reward $\bm{X}_{k, n_k}$ will be returned after selecting child $v_k$.
Successive selections of child $v_k$ yield rewards $\bm{X}_{k, 1}, \bm{X}_{k, 2}, \dots$, which are drawn from an unknown distribution with unknown expected reward $\bm{\mu}_k$.
A policy is an algorithm that chooses a child node based on the sequence of past selections and obtained rewards.
The goal of a policy is to minimize the number of times choosing a sub-optimal node.
\end{problem}

In Pareto MCTS, node selection only happens after all child nodes have been expanded.
In other words, there is an initialization step
in which each node has been selected once.
For easy reference, we summarize the node selection policy below.

\begin{policy}\label{policy}
Given a node selection problem as Problem \ref{node_selection}, choose each child node once in the first $K$ steps.
After that, build an approximate Pareto optimal node set based on the following upper confidence bound vector:
\begin{equation}\label{eq:PUCB}
    \bm{U}(k) = \bar{\bm{X}}_{k, n_k} + \sqrt{\frac{4\ln{n} + \ln{D}}{2 n_k}},
\end{equation}
\end{policy}
\noindent where $K$ is the number of child nodes, $n_k$ is the number of times child $k$ has been selected so far, $\bar{\bm{X}}_{k, n_k}$ is the average reward obtained from child $k$, $D$ is the number of dimensions of the reward, and $n = \sum^K_{k=1} n_k$.

In the following proof, we shall use the concept of most dominant optimal node originated from the $\epsilon$-dominance concept~\cite{kollat2008new} of multi-objective optimization.
Intuitively, the most dominant optimal node of a given node is the one in the (estimated) Pareto optimal set which is the ``farthest away'' from the given node.
\begin{definition}[Most Dominant Optimal Node]
Given a node $v_k$ and a node set $\mathcal{V}$ such that $\forall v_{k'} \in \mathcal{V}, v_{k'} \succ v_{k}$.
For all $v_{k'} \in \mathcal{V}$, there exists exactly one minimum positive constant $\epsilon_{k'}$ such that
$$\epsilon_{k'} = \min \{\epsilon | \exists d \in \{1, 2, \dots, D\} \text{ s.t. } \mu'_{k,d} + \epsilon > \mu_{k,d}\}.$$

Let the index of the maximum $\epsilon_{k'}$ be $k^*$,
$$k^* = \argmax_{k'} \epsilon_{k'},$$
then the most dominant optimal node is $v_{k^*}$.
\end{definition}
Throughout the paper, symbols related to the most dominant optimal node will be indexed by a star($^*$).
As in \cite{kocsis2006bandit}, we allow the expected average rewards to drift as a function of time and our main assumption is that it will converge pointwise.
Here we introduce two assumptions so that the later proof can exploit.

\begin{assumption}[Convergence of Expected Average Rewards]\label{convergence_expected_average}
The expectations of the average rewards $\mathbb{E}[\bar{\bm{X}}_{k, n_k}]$ converge pointwise to the limit $\bm{\mu}_{k}$ for all child nodes:
\begin{equation}
    \bm{\mu}_{k} = \lim_{n_k \rightarrow \infty} \mathbb{E}[\bar{\bm{X}}_{k, n_k}].
\end{equation}
\end{assumption}

For a sub-optimal node $v_k$ and its most dominant optimal node $v_{k^*}$ from $\mathcal{P^*}$, we define $\bm{\Delta}_k = \bm{\mu}^* - \bm{\mu}_k$ to denote their difference. 
\begin{assumption}\label{filtration_assumption}
Fix $1\leq k \leq K$ and $1\leq d \leq D$. Let $\{\mathcal{F}_{k, t, d}\}_t$ be a filtration such that $\{X_{k, t, d}\}_t$ is $\{\mathcal{F}_{k, t, d}\}$-adapted and $X_{k, t, d}$ is conditionally independent of $\mathcal{F}_{k, t+1, d}, \mathcal{F}_{k, t+2, d}, \dots$ given $\mathcal{F}_{k, t-1, d}$.
\end{assumption}

For the sake of simplifying the notation, we define $\bm{\mu}_{k, n_k} = \mathbb{E}[\bar{\bm{X}}_{k, n_k}]$ and $\bm{\delta}_{k, n_k} = \bm{\mu}_{k, n_k} - \bm{\mu}_{k}$ as the residual for the drift.
Clearly, $\lim_{n_k \rightarrow \infty} \bm{\delta}_{k, n_k} = \bm{0}$.
By definition, $\forall \xi > 0$, $\exists N_0(\xi)$ such that $\forall n_k \geq N_0(\xi)$, $\forall d \in \{1, 2, \dots, D\}$, $|\delta_{k, n_k, d}| \leq \xi \Delta_{k, d} / 2$.
We present the first theorem below.

\begin{theorem}\label{suboptimal_bound}
Consider Policy \ref{policy} applied to the node selection Problem \ref{node_selection}.
Suppose Assumption \ref{convergence_expected_average} is satisfied.
Let $T_k(n)$ denote the number of times child node $v_k$ has been selected in the first $n$ steps.
If child node $v_k$ is a sub-optimal node (i.e. $v_k \not \in \mathcal{P}^*$), then $\mathbb{E}[T_k(n)]$ is logarithmically bounded:
\begin{equation}
    \mathbb{E}[T_k(n)] \leq \frac{8 \ln{n} + 2 \ln{D}}{(1-\xi)^2 (\min\limits_{k, d}\Delta_{k,d})^2} + N_0(\xi) + 1 + \frac{\pi^2}{3}.
\end{equation}
\end{theorem}
\begin{proof}
Proof is provided in Appendix A.
\end{proof}

The following lemma gives a lower bound on the number of times each child node being selected.
\begin{lemma}\label{lower_bound}
There exists positive constant $\rho$ such that $\forall k, n, T_k(n) \geq \ceil{\rho \log(n)}$.
\end{lemma}

The upcoming lemma states that the average reward will concentrate around its expectation after enough node selection steps.
\begin{lemma}[Tail Inequality]\label{tail_inequality}
Fix arbitrary $\eta > 0$ and let $\sigma = 9\sqrt{\frac{2\ln(2/\eta)}{n}}$.
There exists $N_1(\eta)$ such that $\forall n \geq N_1(\eta), \forall d \in \{1, \dots, D\}$, the following bounds hold true:
\begin{align}
    \mathbb{P}(\bar{X}_{n,d} \geq \mathbb{E}[\bar{X}_{n,d}] + \sigma) \leq \eta,\\
    \mathbb{P}(\bar{X}_{n,d} \leq \mathbb{E}[\bar{X}_{n,d}] - \sigma) \leq \eta.
\end{align}
\end{lemma}
Correctness of Lemma \ref{lower_bound} and Lemma \ref{tail_inequality} is provided in~\cite{kocsis2006bandit}.

\begin{theorem}[Convergence of Failure Probability]\label{thm:optimality}
Consider the node selection policy described in Algorithm \ref{alg:pareto_best_child} applied to the root node.
Let $I_t$ be the selected child node and $\mathcal{P}^*$ be the Pareto optimal node set.
Then,
\begin{equation}
    \mathbb{P}(I_t \not \in \mathcal{P}^*) \leq Ct^{-\frac{\rho}{2}\left(\frac{\min\limits_{k, d} \Delta_{k, d}}{36}\right)^2},
\end{equation}
with some constant $C$.
In particular, it holds that $\lim_{t\rightarrow \infty} \mathbb{P}(I_t\not \in \mathcal{P}^*) = 0$
\end{theorem}
\begin{proof}
Proof is provided in Appendix B.
\end{proof}

Theorem~\ref{thm:optimality} shows that, at the root node, the probability of choosing a child node (and corresponding action) which is not in the Pareto optimal set converges to zero at a polynomial rate as the number of node selection grows.


\begin{figure} [t] \vspace{-15pt} 
    \centering
  \subfloat[\label{fig:exmaple_env}]{%
      \includegraphics[width=0.45\linewidth]{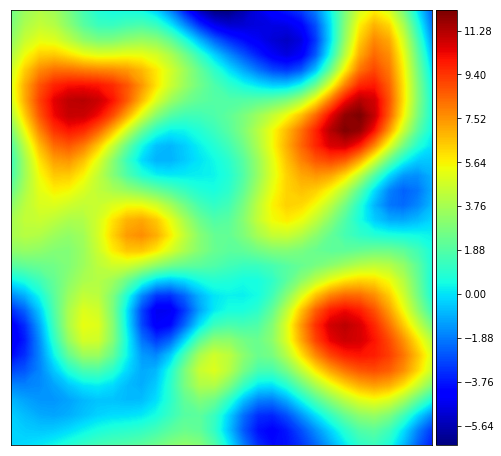}}
      \quad
  \subfloat[\label{fig:example_path}]{%
        \includegraphics[width=0.45\linewidth]{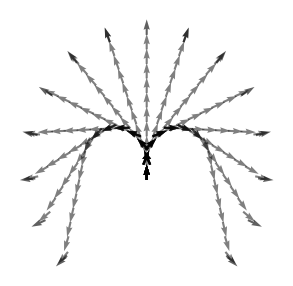}}
    \caption{(a) Environment with hotspots used as ground truth. The heat map represents level of interest. (b) Primitive paths for a robot.\vspace{-15pt} }
    \label{fig:env_path}
\end{figure}

\section{EXPERIMENTS}

To thoroughly evaluate our framework, we have compared different methods in extensive simulations using both synthetic data and real-world data, where the basic scenario is the hotspot monitoring task via informative planning for a robot.
The ideal behavior for the robot is to first explore the environment to discover the hotspots and then exploit these important areas to collect more valuable samples.
This is a bi-objective task although our algorithm is suitable for multi-objective tasks in general.
We choose this scenario for comparison (and illustration) purpose, since one of the comparing methods can only handle bi-objective case. In addition, hotspot monitoring task can be easily visualized for interpretation.


We have compared our algorithm with two other baseline methods.
The first method is the Monte Carlo tree search with information-theoretic objective, which has been successfully applied to planetary exploration mission~\cite{arora2017mcts}, environment exploration~\cite{best2019dec, corah2017efficient}, and monitoring spatiotemporal process~\cite{marchant2014sequential}. We called the method {\em information MCTS}.
The second method is a upper confidence bound based online planner~\cite{sun2017no} which balances exploration and exploitation in a near-optimal manner with appealing no-regret properties.
However, when choosing an optimal trajectory, only the primitive paths of current state are considered in their model.
For comparison, we modify and extend this model to MCTS by using the upper confidence bound (see section \RN{4} of \cite{sun2017no}) as the reward function of MCTS, which is called {\em UCB MCTS}.


\begin{figure*} \vspace{-5pt}  
    \centering
    \subfloat[\label{fig:info_path_syn}]{\includegraphics[height=1.4in]{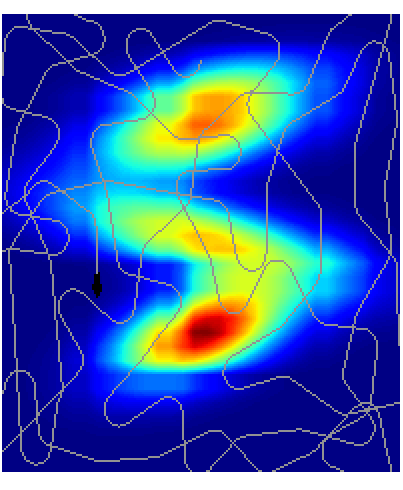}}
    \subfloat[\label{fig:info_mean_syn}]{\includegraphics[height=1.4in]{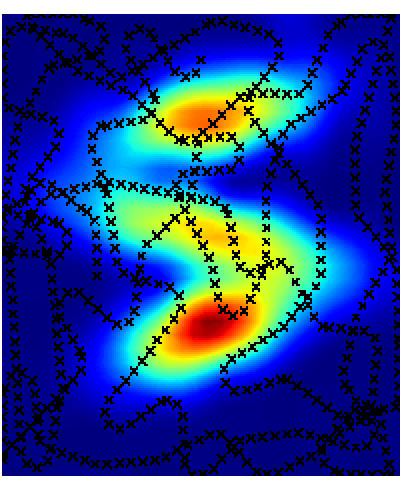}} \  
    \subfloat[\label{fig:ucb_path_syn}]{\includegraphics[height=1.4in]{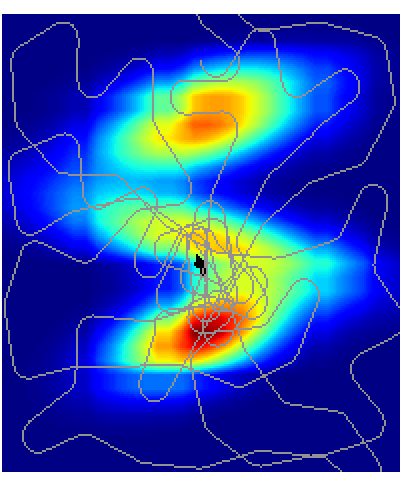}}
    \subfloat[\label{fig:ucb_mean_syn}]{\includegraphics[height=1.4in]{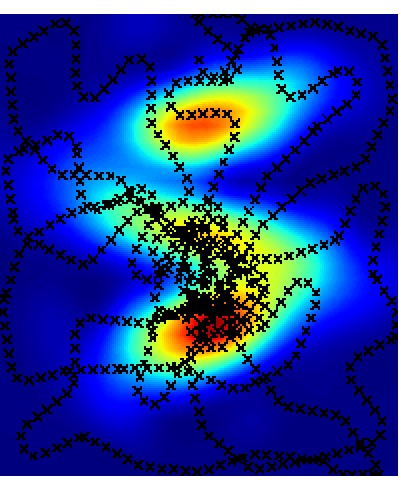}} \ 
    \subfloat[\label{fig:pareto_path_syn}]{\includegraphics[height=1.4in]{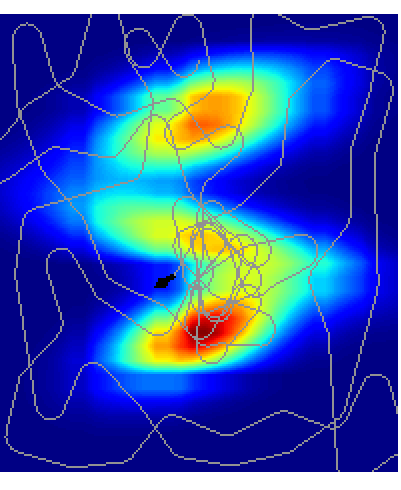}}
    \subfloat[\label{fig:pareto_mean_syn}]{\includegraphics[height=1.4in]{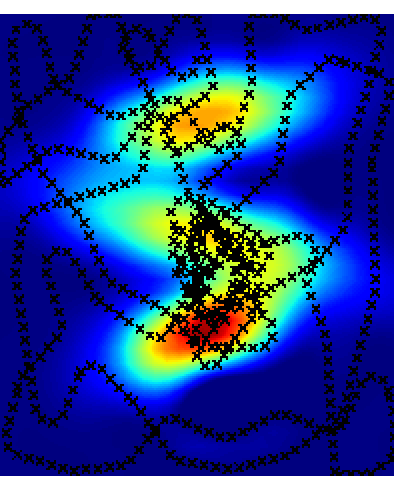}}
    \caption{The arrow represents the robot. (a)(c)(e) The line shows the resulting path of each algorithm with the underlying environment as the background. (b)(d)(f) Robot's estimation of the target value and collected samples. Red represents high value and blue indicates low value. (a)(b) Information MCTS. (c)(d) UCB MCTS. (e)(f) Pareto MCTS.\vspace{-15pt} }
    \label{fig:path_mean_syn}
\end{figure*}

\begin{figure} \vspace{-5pt} 
    \centering
    \subfloat[]{\includegraphics[width=0.45\linewidth]{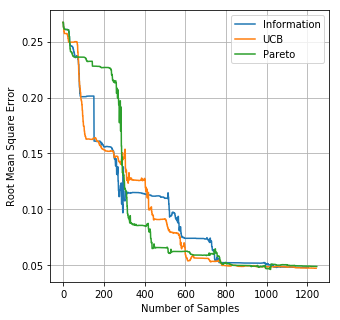}}
    \quad
    \subfloat[]{\includegraphics[width=0.45\linewidth]{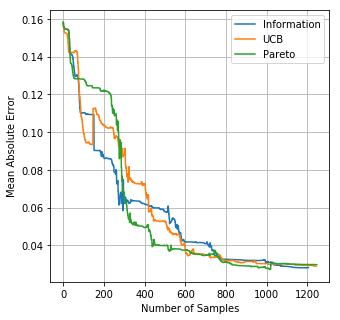}} \vspace{-10pt}
    \subfloat[]{\includegraphics[width=0.45\linewidth]{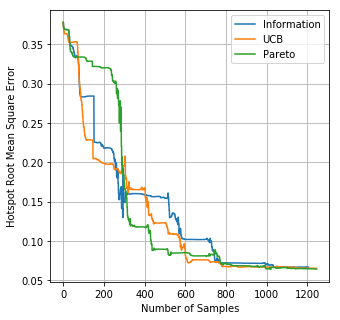}}
    \quad
    \subfloat[]{\includegraphics[width=0.45\linewidth]{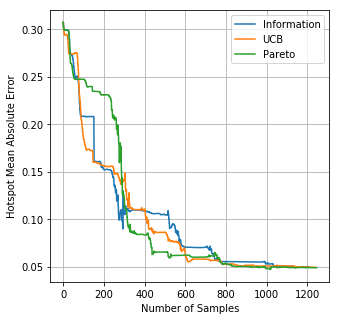}}
    
    \caption{The resulting error of the three algorithm in the synthetic problem. (a) Root mean square error. (b) Mean absolute error. (c) Hotspot root mean square error. (d) Hotspot mean absolute error. \vspace{0pt}
    }
    \label{fig:error_syn}
\end{figure}

\textbf{Rewards: }
As for the reward function, we choose variance reduction as the information-theoretic objective as in~\cite{binney2012branch, hollinger2014sampling}.
The other reward is devised as adding up the predicted values of samples along the path, which encourages the robot to visit high-value areas more frequently.

\textbf{Metrics: }
In the informative planning context, our goal is to minimize the root mean square error (RMSE) between the estimated environmental state (using GP prediction) and the ground truth.
However, in our hotspot monitoring task, we are more concerned with the modeling errors in high-value areas than the entire environment.
Therefore, a {\em hotspot RMSE} is employed for better evaluating the ``exploitation" performance.
We classify the areas with target values higher than the median as hotspots and calculate the RMSE within these hotspots.
In addition, larger errors in unimportant areas are acceptable in this task and RMSE tends to penalize larger errors in a uniform way.
Therefore, we also evaluate the methods using mean absolute error (MAE).
Similarly, we introduced hotspot MAE to highlight algorithms' performance in the important areas.
Last but not least, the {\em percentage of samples} in hotspots measures the quality of the collected data. Ideally, the robot should locate the important areas as soon as possible and gather more valuable samples in these areas.
Also, if there are multiple hotpots, the robot should visit as many hotspots as possible instead of getting stuck in one specific area.


\subsection{Synthetic Problems}
The robot is tasked to monitor several hotspots in  an unknown $10\text{ km }\times 10\text{ km }$ environment, illustrated in Fig.~\ref{fig:exmaple_env}.
The hotspots to be monitored are specified by three Gaussian sources with random placement and parameters.
At each position, the robot has $15$ Dubins paths~\cite{lavalle2006planning} as its primitive paths.
Fig.~\ref{fig:example_path} illustrates an example of available primitive paths.

\begin{figure}[htbp] \vspace{-10pt}
    \centering
  \subfloat[\label{fig:cdom_raw}]{%
        \includegraphics[height=1.3in]{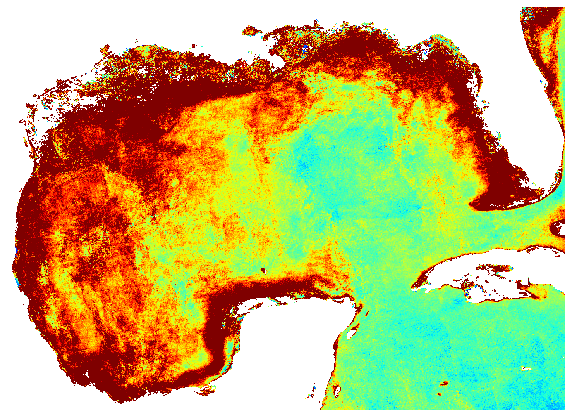}} \quad
  \subfloat[\label{fig:cdom_crop}]{%
      \includegraphics[height=1.3in]{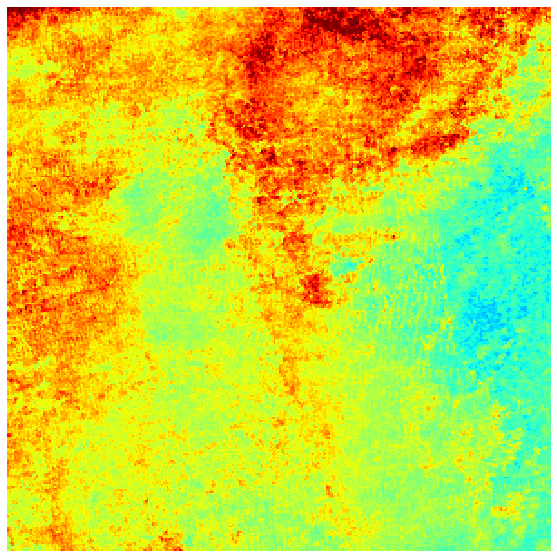}}
    \caption{(a) CDOM raw data of Gulf of Mexico. (b) A cropped region near Louisiana Wetlands. \vspace{-15pt} }
    \label{fig:cdom}
\end{figure}

Fig. \ref{fig:path_mean_syn} shows the the ground truth and prediction with the robot's path and collected samples.
As expected, the information MCTS tends to cover the whole space and collect the samples uniformly, because less information will be gained from a point once it has been observed.
On the contrary, UCB MCTS spent a small amount of effort exploring the environment during the initial phase in order to reduce the uncertainty of the hotspot estimation.
After that, it tends to greedily wander around the high-value areas.
This phenomenon is consistent with the behavior of the UCB algorithm in the multi-armed bandit problem in which the best machine will be played much more times than sub-optimal machines.
We noticed that the bias term in UCB-replanning~\cite{sun2017no} also increases with the mission time, which means that, given enough time, the robot will still try to explore other areas.
However, in our experiments, the task duration could not be set too long due to the scalability of the GP.
Pareto MCTS simultaneously optimize all the objectives.
As a result, more samples can be found in the upper hotspot (see Fig. \ref{fig:pareto_mean_syn} and Fig. \ref{fig:ucb_mean_syn}).

Fig. \ref{fig:error_syn} presents the global error and hotspot error.
The difference between the global error and the hotspot error is negligible because the most important variability is concentrated in the hotspot areas.
We notice that, in the first $200$ samples, the performance of Pareto MCTS is inferior to other two methods.
In fact, this is because Pareto MCTS has not yet discovered any particular hotspot in the initial exploration phase.
Once it finds the a hotspot and starts exploiting that hotspot, the error curve drops drastically, surpassing the other two methods at about $300$ samples.
This property is consistent with our motivation. It is also obvious that our method in general has the steepest error reduction rate, which is particularly important in monitoring highly dynamic environment.

\begin{figure} \vspace{-5pt} 
    \centering
    \subfloat[\label{fig:info_path}]{\includegraphics[width=0.4\linewidth]{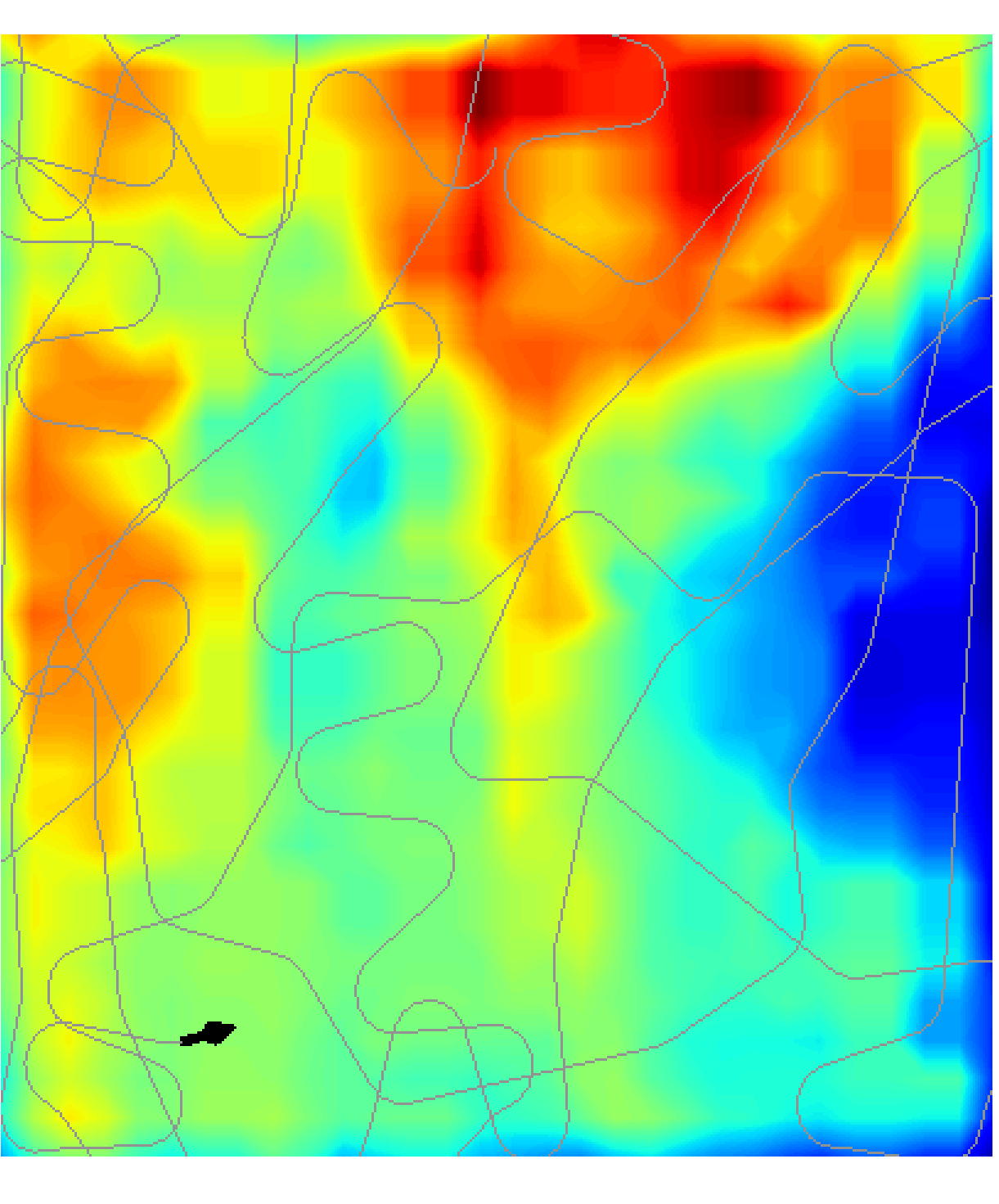}} \quad
    \subfloat[\label{fig:info_mean}]{\includegraphics[width=0.4\linewidth]{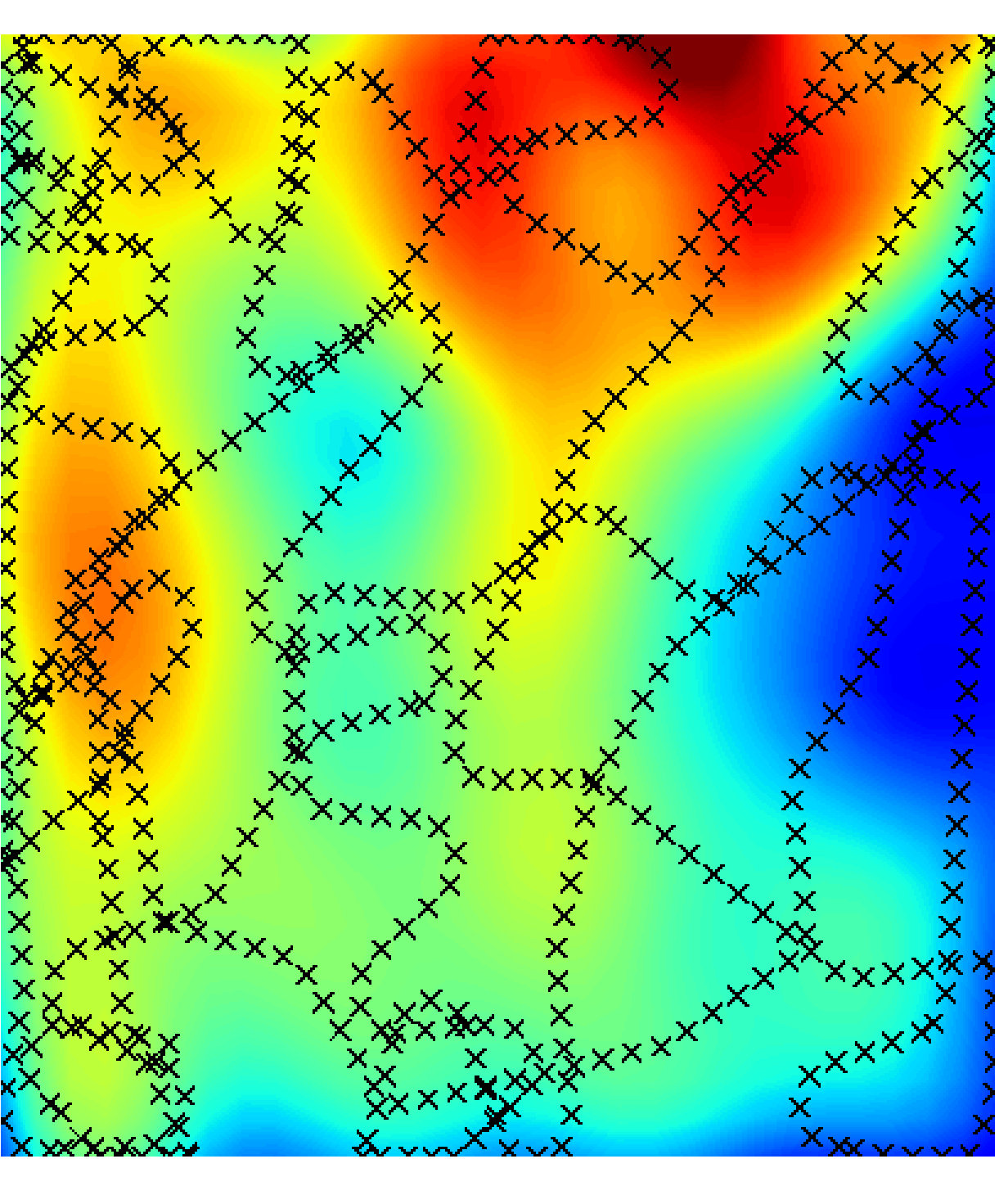}} \vspace{-10pt} \\
    \subfloat[\label{fig:ucb_path}]{\includegraphics[width=0.4\linewidth]{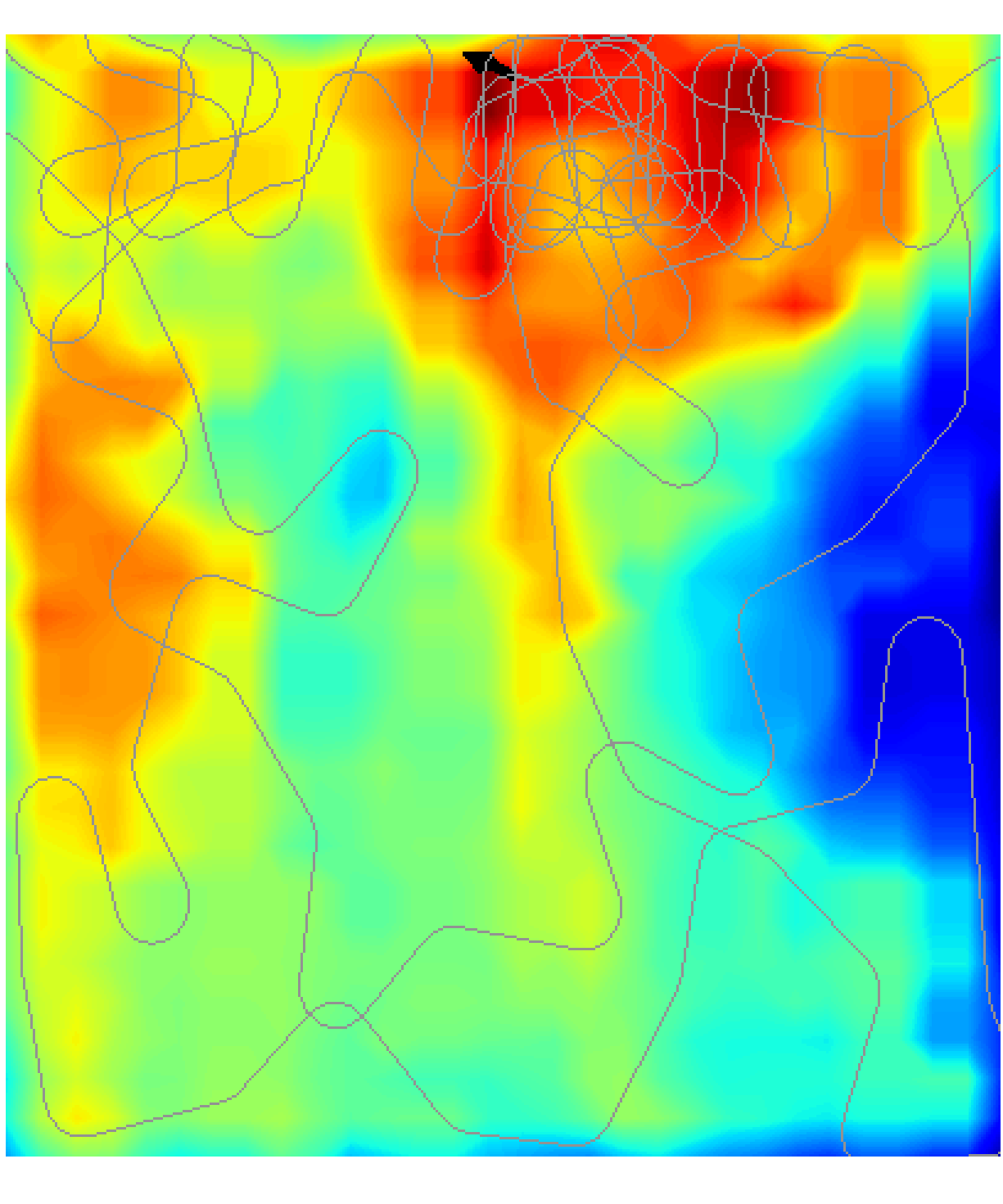}} \quad
    \subfloat[\label{fig:ucb_mean}]{\includegraphics[width=0.4\linewidth]{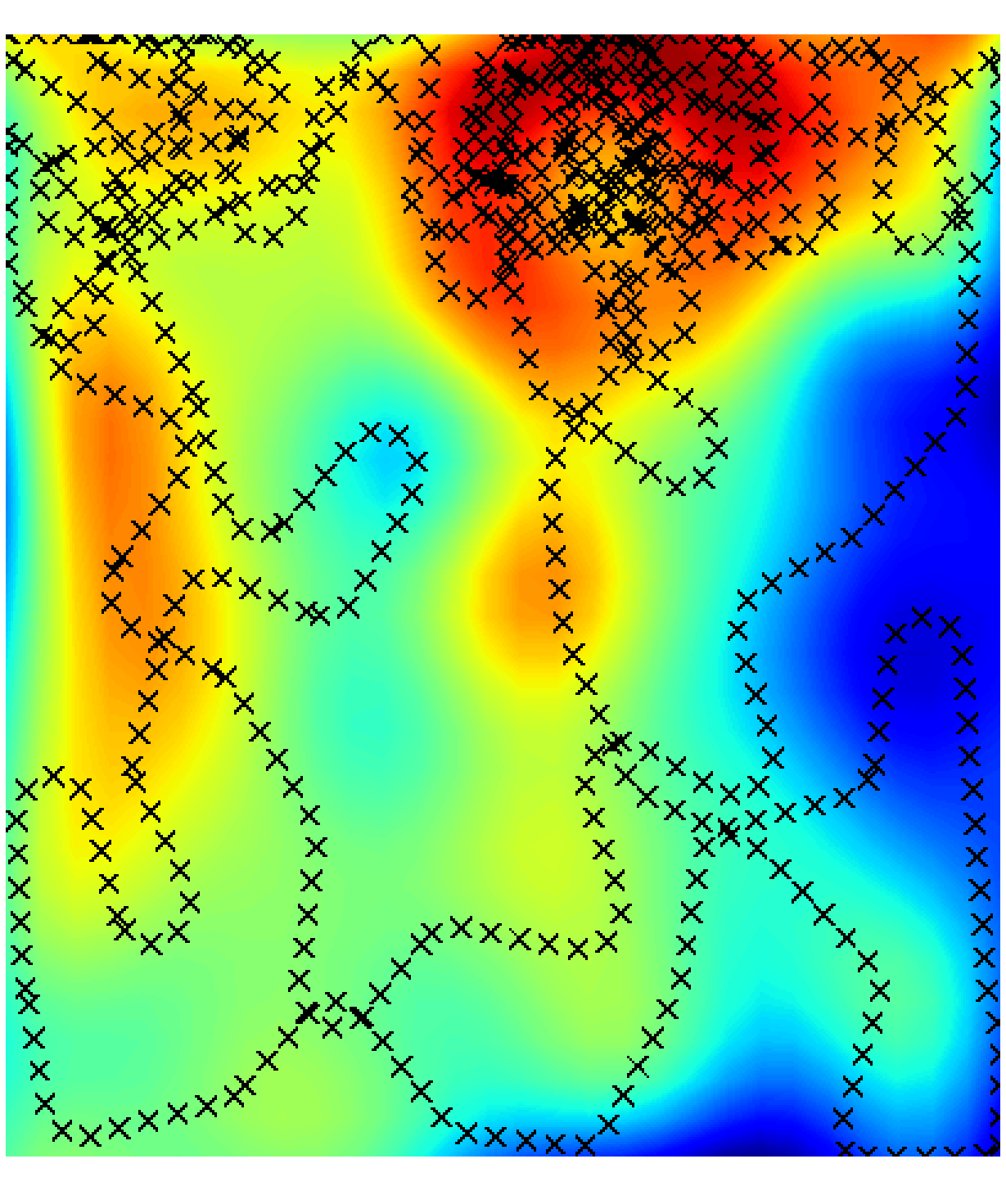}}\vspace{-10pt}\\
    \subfloat[\label{fig:pareto_path}]{\includegraphics[width=0.4\linewidth]{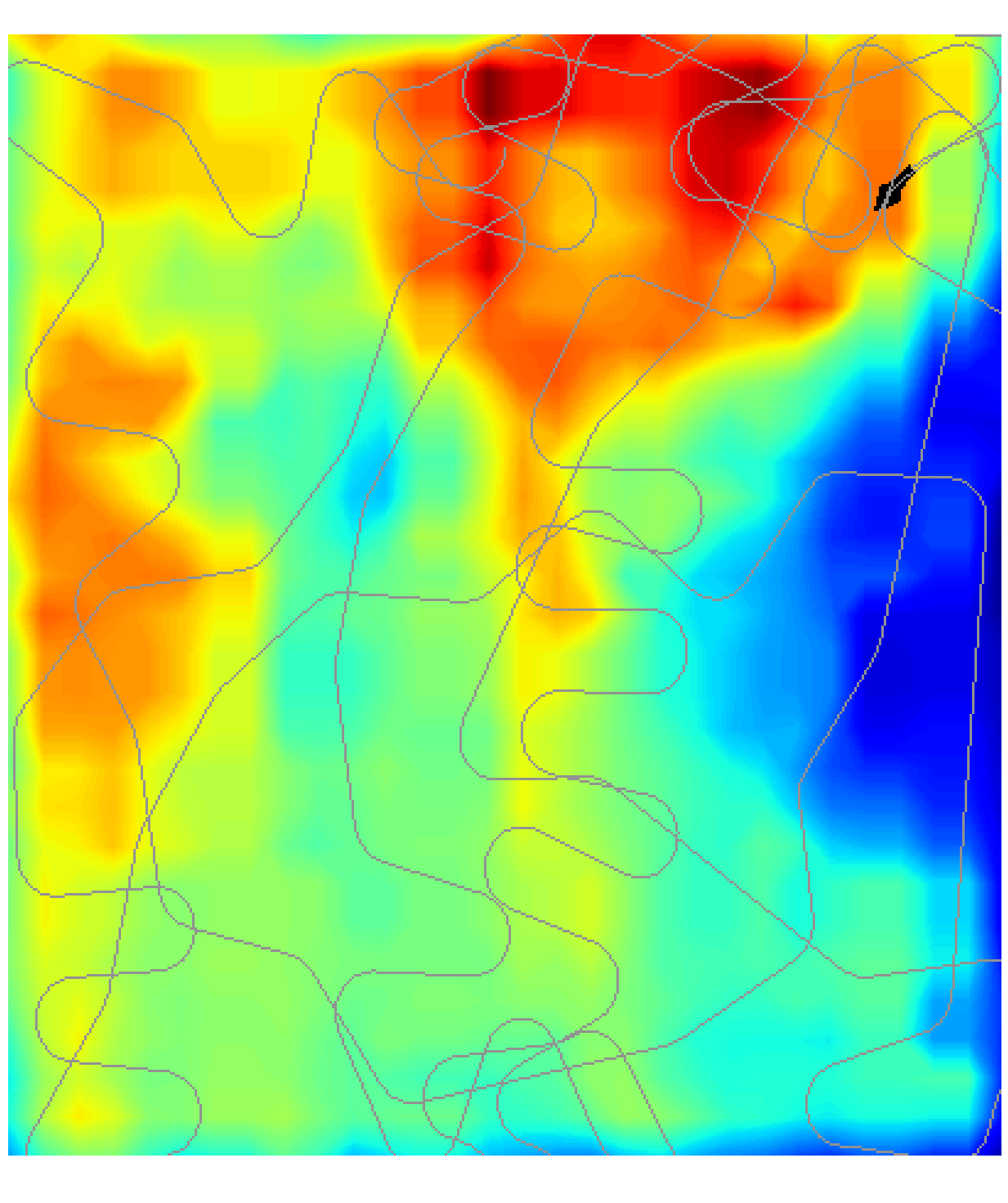}} \quad
    \subfloat[\label{fig:pareto_mean}]{\includegraphics[width=0.4\linewidth]{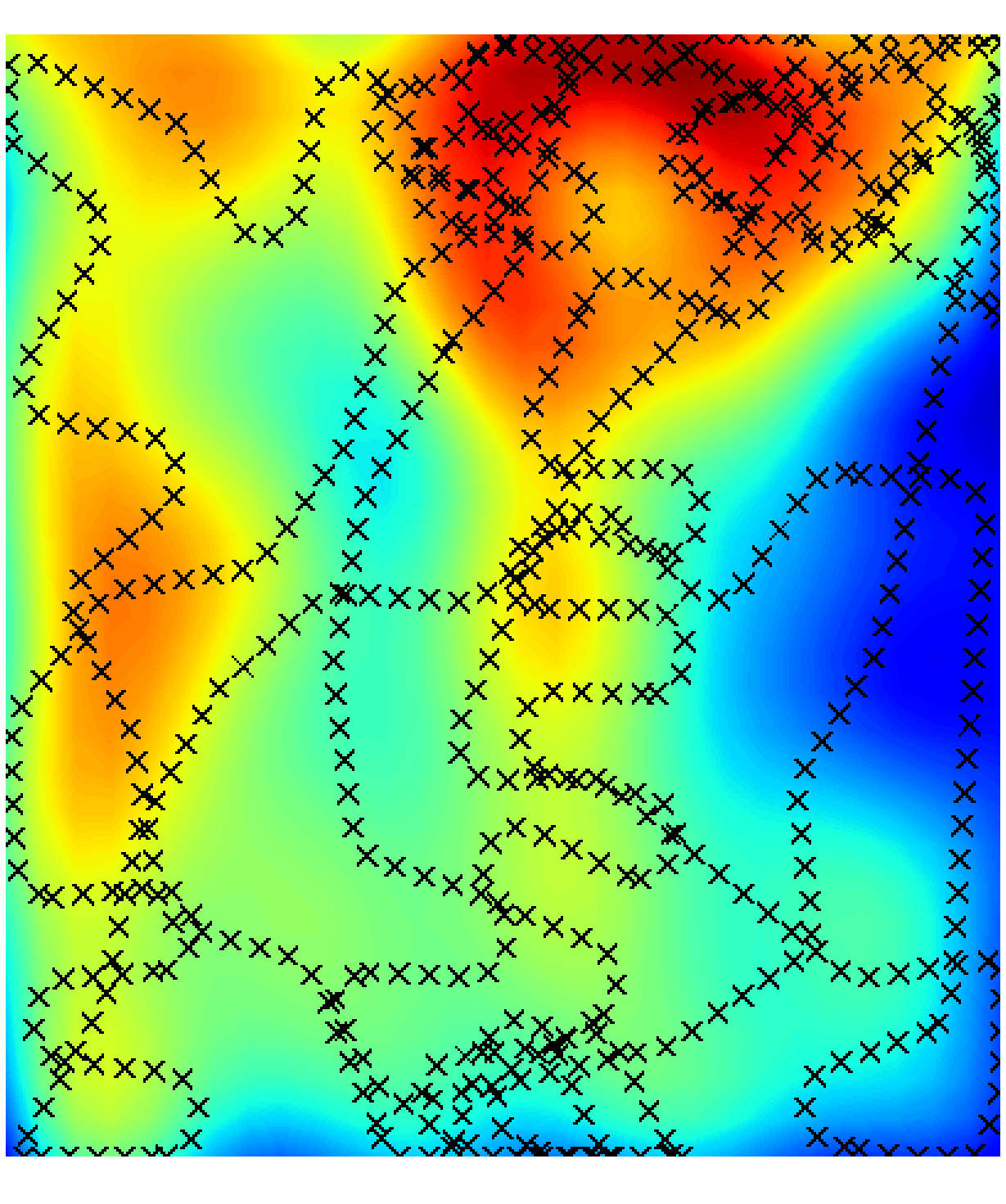}}
    \caption{The resulting path, collected samples, and prediction of each algorithm in the chromophoric dissolved organic material monitoring problem. Blue and red correspond to low and high value, respectively. The arrow represents the robot and the crosses are the sampled data. (a) Path and the ground truth of information MCTS. (b) Collected samples and the estimated hotspot map from information MCTS. (c) Path and the ground truth of UCB MCTS. (d) Collected samples and estimation of UCB MCTS. (e) Path and the ground truth of Pareto MCTS. (f) Collected samples and estimation of Pareto MCTS. \vspace{-15pt}
    }
    \label{fig:result_cdom}
\end{figure}

\subsection{Chromophoric Dissolved Organic Material Monitoring}

We now demonstrate our proposed approach using the chromophoric dissolved organic material (CDOM) data at the Gulf of Mexico, provided by National Oceanic and Atmospheric Administration (NOAA).
The concentration of CDOM has a significant effect on biological activity in aquatic systems.
Very high concentrations of CDOM can affect photosynthesis and inhibit the growth of phytoplanktons.
Fig.~\ref{fig:cdom_raw} is the raw data of CDOM.
We have cropped a smaller region with a higher variability of the target value (Fig. \ref{fig:cdom_crop}).
Due to the scalability issue of the GP, the raw data ($300\times 300$ grids) is down-sampled to ($30\times 30$).


Fig. \ref{fig:result_cdom} reveals similar sample distribution patterns as in the synthetic data.
Specifically, the information MCTS features good spatial coverage.
UCB MCTS prefers to stay at the hotspot with highest target value after a rough exploration of the environment.
Pareto MCTS tends to compromise between hotspot searching and hotspot close examination.
As a result, it exhibits interesting winding paths in some important areas (see Fig. \ref{fig:pareto_mean}).
This allows the robot to collect more samples in those areas without losing too much information gain.

In this experiment, we also show how to incorporate prior knowledge in the Pareto MCTS.
In robotic environmental monitoring, the robot needs to explore the environment extensively in the early stage.
To this end, we always choose the most informative action from the Pareto optimal set at the beginning (first $400$ samples), which makes Pareto MCTS degenerate to information MCTS.
As shown in Fig. \ref{fig:error}, the root mean square error of information MCTS (blue line) and that of Pareto MCTS (green line) visually overlap.
This implies that there is almost no loss in global modeling error.
At the same time, the percentage of samples collected from hotspots has increased.
Fig. \ref{fig:hotspot_percent} shows the percentage of hotspot samples.
These results validate the benefits of multi-objective informative planning and Pareto MCTS.

\begin{figure}\vspace{0pt} 
    \centering
    \subfloat[]{\includegraphics[width=0.45\linewidth]{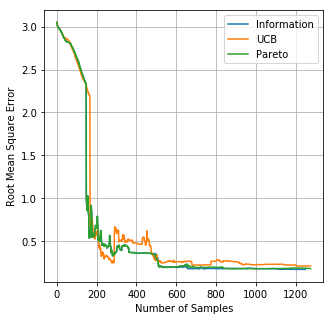}}
    \quad 
    \subfloat[]{\includegraphics[width=0.45\linewidth]{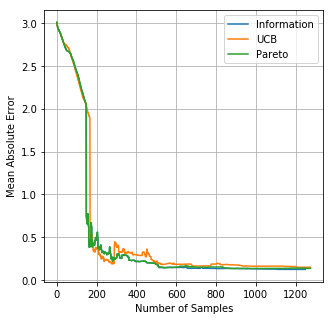}}
    
    \caption{The root mean square error and mean absolute error of the three algorithms in the chromophoric dissolved organic material monitoring problem. (a) Root mean square error. (b) Mean absolute error. The blue line is visually overlapped with the green line. They are separated after zooming in, but the difference is negligible. 
    }
    \label{fig:error}
\end{figure}

\begin{figure} \vspace{0pt} 
    \centering
    {\includegraphics[width=0.9\linewidth, height=3cm]{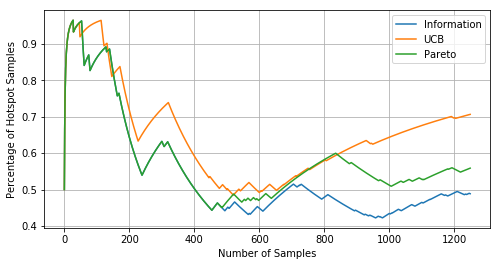}}
    \caption{Percentage of hotspot samples of the three algorithms in the chromophoric dissolved organic material monitoring problem.
    \vspace{-15pt}
    } \label{fig:hotspot_percent}
\end{figure}


\section{CONCLUSION}
This paper presents a Pareto multi-objective optimization based informative planning approach.
We show that the searching result of Pareto MCTS converges to the Pareto optimal actions at a polynomial rate, and the number of times choosing a sub-optimal node in the course of tree search has a logarithmic bound.
Our method allows the robot to adapt to the target environment and adjust its concentrations on environmental exploration versus exploitation based on its  knowledge (estimation) of the environment.
We validate our approach in a hotspot monitoring task using real-world data and the results reveal that our algorithm enables the robot to explore the environment and, at the same time, visit hotspots of high interests more frequently.


\section*{APPENDIX}
\subsection{Proof to Theorem \ref{suboptimal_bound}}\label{proof_thm_1}
Let variable $I_t$ be the index of the selected child node at decision step $t$ and $\bm{1}\{\cdot\}$ be a Boolean predicate function.
Let the bias term in Eq.~\eqref{eq:PUCB} be $c_{t, s} = \sqrt{\frac{4\ln{t} + \ln{D}}{2s}}$.
Then $\forall l > 0 \text{ and } l \in \mathbb{Z}^+$, for any sub-optimal node $v_k$, we have the upper bound $T_k(n)$.
\begin{align*}
    T_k(n) =& l + \sum^n_{t=K+1} \bm{1} \{I_t = k\}\\
        \leq& l + \sum^n_{t=K+1} \bm{1} \{I_t = k, T_k(t-1) \geq l\}\\
        \leq& l + \sum^n_{t=K+1} \bm{1} \{\bar{\bm{X}}^*_{T_j(t-1)} + c_{t-1, T^*(t-1)} \\
        &\nsucc \bar{\bm{X}}_{k, T_k(t-1)} + c_{t-1, T_k(t-1)}, T_k(t-1) \geq l\}\\
        \leq& l + \sum^\infty_{t=1} \sum^{t-1}_{s=1} \sum^{t-1}_{s_k=l} \bm{1}\{\bar{\bm{X}}^*_s + c_{t, s} \nsucc \bar{\bm{X}}_{k, s_k} + c_{t, s_k}\}\\
\end{align*}
$\bar{\bm{X}}^*_s + c_{t, s} \nsucc \bar{\bm{X}}_{k, s_k} + c_{t, s_k}$ implies that at least one of the following must hold:
\begin{align}
        \bar{\bm{X}}^*_s &\nsucc \bm{\mu}^*_s - c_{t, s}\label{eq:first_tail}\\
        \bar{\bm{X}}_{k, s_k} &\nprec \bm{\mu}_{k, s_k} - c_{t, s}\label{eq:second_tail}\\
        \bm{\mu}^*_s &\nsucc \bm{\mu}_{k, s} + 2c_{t, s_k}\label{eq:impossible}
\end{align}
Otherwise, If Eq.~\eqref{eq:first_tail}, \eqref{eq:second_tail} are false, then Eq.~\eqref{eq:impossible} is true.

We bound the probability of events Eq.~\eqref{eq:first_tail} \eqref{eq:second_tail} using Chernoff-Hoeffding Bound and Union Bound.

\newcommand\unionBoundEq{\mathrel{\overset{\makebox[0pt]{\mbox{\normalfont\tiny\sffamily Union Bound}}}{=}}}
\begin{align*}
    &\mathbb{P}(\bar{\bm{X}}^*_s \nsucc \bm{\mu}^*_s - c_{t, s})\\
 = &\mathbb{P}\left( (\bar{X}^*_{s,1} < \mu^*_{s,1} - c_{t, s}) \lor \cdots \lor (\bar{X}^*_{s,D} < \mu^*_{s,D} - c_{t, s}) \right)\\
 \leq& \sum^D_{d=1} \mathbb{P} \left( \bar{X}^*_{s,d} < \mu^*_{s,d} - c_{t, s} \right) \tag{Union Bound}\\
 \leq& \sum^D_{d=1} \frac{1}{D} t^{-4} = t^{-4}\tag{Chernoff-Hoeffding Bound}
\end{align*}
Similarly, $\mathbb{P}(\bar{\bm{X}}_{i, s_k} \nprec \bm{\mu}_{i, s_k} - c_{t, s}) \leq t^{-4}$

Let $$l = \max\left\{\ceil[\bigg]{\frac{8\ln{t} + 2\ln{D}}{(1-\xi)^2 \min\limits_{k,d} \Delta^2_{k, d}}}, N_0(\xi)\right\}.
$$ Since $s_{k, d} \geq l_0$, (\ref{eq:impossible}) is false.

Therefore, plugging the above results into the bound on $T_k(n)$ and taking expectations of both sides, we get
\begin{align*}
    \mathbb{E}[T_k(n)] &\leq \ceil[\bigg]{\frac{8\ln{t}+2\ln{D}}{(1-\xi)^2 \min\limits_{k, d} \Delta^2_{k, d}}} + N_0(\xi)+ \sum^\infty_{t=1} \sum^{t-1}_{s=1} \sum^{t-1}_{s_k=l}\\
    &\left( \mathbb{P}(\bar{\bm{X}}^*_s \nsucc \bm{\mu}^*_s - c_{t, s}) + \mathbb{P}(\bar{\bm{X}}_{k, s_k} \nprec \bm{\mu}_{k, s_k} - c_{t, s}) \right)\\
     \leq& \frac{8\ln{t} + 2\ln{D}}{(1-\xi)^2 \min\limits_d \Delta^2_{k, d}} + N_0(\xi) + 1 + \frac{\pi^2}{3},
\end{align*}
which concludes the proof.

\subsection{Proof to Theorem \ref{thm:optimality}}
Let $k$ be the index of a sub-optimal node.
Then $\mathbb{P}(I_t \not \in \mathcal{P}^*) \leq \sum_{v_k \not \in \mathcal{P}^*} \mathbb{P}\left(\bar{\bm{X}}_{k, T_k(t)} \nprec \bar{\bm{X}}^*_{T^*(t)}\right)$.
Note that $\bar{\bm{X}}_{k, T_k(t)} \nprec \bar{\bm{X}}^*_{T^*(t)}$ implies 
\begin{equation}\label{failure_first}
\bar{\bm{X}}_{k, T_k(t)} \nprec \bm{\mu}_k + \frac{\bm{\Delta}_k}{2}, 
\end{equation}
or 
\begin{equation}\label{failure_second}
\bar{\bm{X}}^*_{T^*(t)} \nsucc \bm{\mu}^* + \frac{\bm{\Delta}_k}{2}.
\end{equation}
Otherwise, suppose Eq.~\eqref{failure_first} and \eqref{failure_second} do not hold, we have $\bar{\bm{X}}_{k, T_k(t)} \prec \bar{\bm{X}}^*_{T^*(t)}$ which yields a contradiction.
Hence, 
\begin{align*}
&\mathbb{P}\left(\bar{\bm{X}}_{k, T_k(t)} \nprec \bar{\bm{X}}^*_{T^*(t)}\right) \\
\leq& \underbrace{\mathbb{P}(\bar{\bm{X}}_{k, T_k(t)} \nprec \bm{\mu}_k + \frac{\bm{\Delta}_k}{2})}_{\text{first term}}
+ \underbrace{\mathbb{P}(\bar{\bm{X}}^*_{T^*(t)} \nsucc \bm{\mu}^* + \frac{\bm{\Delta}_k}{2})}_{\text{second term}}.
\end{align*}

Here we show how to bound the first term:
\begin{align*}
    &\mathbb{P}(\bar{\bm{X}}_{k, T_k(t)} \nprec \bm{\mu}_k + \frac{\bm{\Delta}_k}{2})\\
    \leq& \sum^D_{d=1} \mathbb{P}(\bar{X}_{k, T_k(t), d} > \mu_{k,d} + \frac{\Delta_{k, d}}{2})\tag{Union Bound}\\
    \leq& \sum^D_{d=1} \mathbb{P}(\bar{X}_{k, T_k(t), d} \geq \mu_{k, T_k(t), d} - \underbrace{|\delta_{k, T_k(t), d}|}_{\text{converges to 0}} + \frac{\Delta_{k, d}}{2})\\
    \leq& \sum^D_{d=1} \mathbb{P}(\bar{X}_{k, T_k(t), d} \geq \mu_{k, T_k(t), d} + \frac{\Delta_{k, d}}{4}) \\
    \leq& \sum^D_{d=1} \text{constant} (\frac{1}{t})^{\frac{\rho}{2}\left( \frac{\min\limits_{k, d}\Delta_{k, d}}{36} \right)^2} \tag{Lemma \ref{tail_inequality}}
\end{align*}
The last step makes use of Lemma \ref{lower_bound} and Lemma \ref{tail_inequality}.

The second term can be bounded in a similar way. Finally, an integration of the bounds shows that the failure probability converges to $0$ at a polynomial rate as the number of selection goes to infinity.


\bibliographystyle{plainnat}
\bibliography{references}

\end{document}